\pdfoutput=1

\documentclass[11pt]{article}

\usepackage[final]{EACL2023}

\usepackage{times}
\usepackage{latexsym}

\usepackage[T1]{fontenc}

\usepackage[utf8]{inputenc}

\usepackage{microtype}

\usepackage{inconsolata}

\usepackage{amsmath,amsfonts,bm}

\def\eqref#1{equation~\ref{#1}}

\def\1{\bm{1}}

\DeclareMathAlphabet{\mathsfit}{\encodingdefault}{\sfdefault}{m}{sl}
\SetMathAlphabet{\mathsfit}{bold}{\encodingdefault}{\sfdefault}{bx}{n}

\newcommand{\R}{\mathbb{R}}

\usepackage{hyperref}
\usepackage{url}
\usepackage{graphicx}
\usepackage{booktabs} 
\usepackage{blindtext}
\usepackage{amsthm}
\usepackage{soul}
\usepackage{float}

\usepackage[shortlabels]{enumitem}
\definecolor{light_blue}{HTML}{DCE6F1}
\usepackage[most]{tcolorbox}
\tcbset{on line, 
        boxsep=1pt, left=0pt,right=0pt,top=0pt,bottom=0pt,
        colframe=white,colback=light_blue,  
        highlight math style={enhanced}
        }

\newcommand\nl[1]{\textit{``#1''}}

\newcommand\idiomem{\textsc{IdioMem}}

\newcommand\magpie{\textsc{MAGPIE}}
\newcommand\lidioms{\textsc{LIdioms}}
\newcommand\efdata{\textsc{EF}}
\newcommand\epie{\textsc{EPIE}}
\newcommand\wikitext{\textsc{WikiText-103}}

\newcommand\lama{\textsc{LAMA}}
\newcommand\lamauhn{\textsc{LAMA-UHN}}

\newcommand\gpt{\textsc{GPT2}}
\newcommand\gptm{\textsc{GPT2}$_{\textsc{M}}$}
\newcommand\gptl{\textsc{GPT2}$_{\textsc{L}}$}

\newcommand\bertb{\textsc{BERT}$_{\textsc{B}}$}
\newcommand\bertl{\textsc{BERT}$_{\textsc{L}}$}

\newcommand\mem{\texttt{mem-idiom}}
\newcommand\unmem{\texttt{non-mem-idiom}}
\newcommand\wiki{\texttt{wiki}}

\definecolor{darkpastelpurple}{rgb}{0.59, 0.44, 0.84}

\title{Understanding Transformer Memorization Recall Through Idioms}

\author{
Adi Haviv$^{\tau}$ \hspace{0.5cm} Ido Cohen$^{\tau}$ \hspace{0.5cm}  Jacob Gidron$^{\tau}$ \hspace{0.5cm} \textbf{Roei Schuster}$^{\mu}$  \\\textbf{Yoav Goldberg}$^{\alpha\beta}$ \hspace{0.5cm}\textbf{Mor Geva}$^{\alpha} \thanks{\;\;~~Now at Google Research.}$ \\ \\
$^{\tau}$Tel Aviv University \hspace{0.04cm}
$^{\mu}$Wild Moose \hspace{0.04cm}
$^{\beta}$Bar-Ilan University \hspace{0.04cm} 
$^{\alpha}$Allen Institute for AI \\
\small{\texttt{adi.haviv@cs.tau.ac.il}}, 
\small{\texttt{roei@wildmoose.ai, pipek@google.com}}, \\
\small{\texttt{\{its.ido, jacob.u.gidron, yoav.goldberg\}@gmail.com}} \\
}

\newtheorem{claim}{Claim}[section]

\begin{document}
\maketitle
\begin{abstract}

To produce accurate predictions, language models (LMs) must balance between generalization and memorization. Yet, little is known about the mechanism by which transformer LMs employ their memorization capacity. When does a model decide to output a memorized phrase, and how is this phrase then retrieved from memory? 
In this work, we offer the first methodological framework for probing and characterizing recall of memorized sequences in transformer LMs. 
First, we lay out criteria for detecting model inputs that trigger memory recall, and propose \emph{idioms} as inputs that typically fulfill these criteria.  Next, we construct a dataset of English idioms and use it to compare model behavior on memorized vs. non-memorized inputs. Specifically, we analyze the internal prediction construction process by interpreting the model's hidden representations as a gradual refinement of the output probability distribution. We find that across different model sizes and architectures, memorized predictions are a two-step process: early layers promote the predicted token to the top of the output distribution, and upper layers increase model confidence. This suggests that memorized information is stored and retrieved in the early layers of the network. 
Last, we demonstrate the utility of  our methodology beyond idioms in memorized factual statements.
Overall, our work makes a first step towards understanding memory recall, and provides a methodological basis for future studies of transformer memorization.\footnote{\label{foot:codedata}Our code and data 
are available at \url{https://github.com/adihaviv/idiomem/}.}

\end{abstract}

\section{Introduction}
\label{sec:intro}

\begin{figure*}[t]
    \centering
    \includegraphics[width=\textwidth]{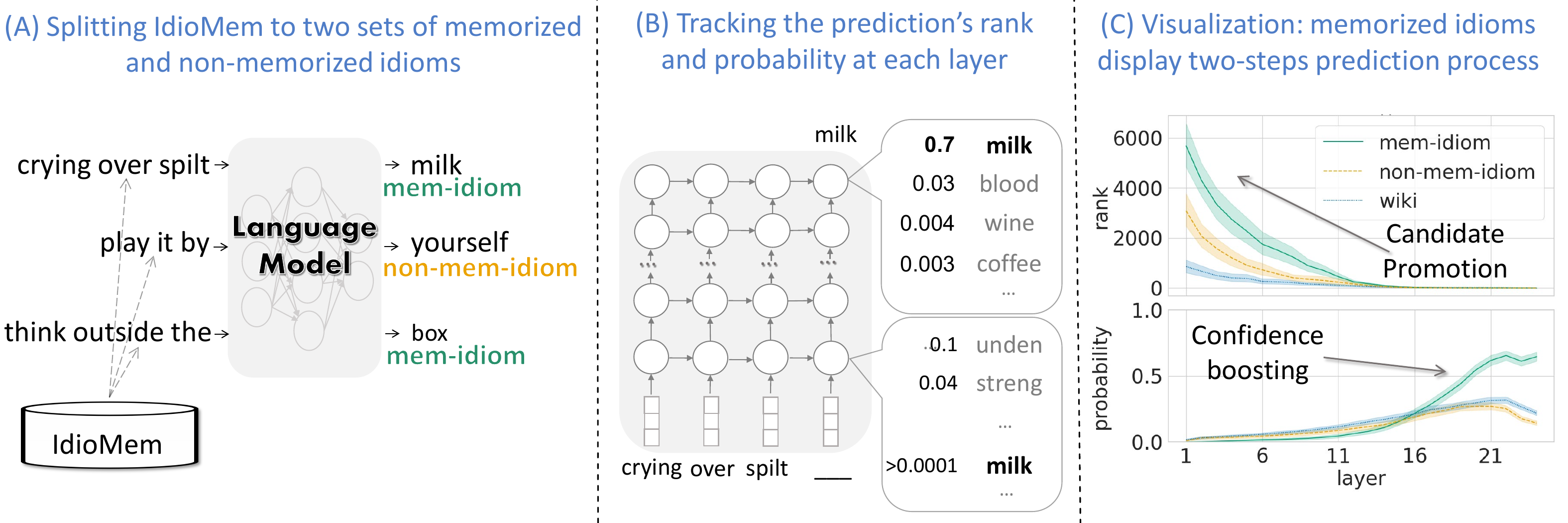}
    \caption{Our methodological framework for probing and analyzing memorized predictions of a given LM: (A) we create two sets of memorized (\texttt{mem-idiom}) and non-memorized (\texttt{non-mem-idiom}) idioms by probing the LM with instances from \idiomem{}, (B) for each instance, we extract hidden features of the prediction computation -- the rank and probability of the predicted token across layers, and (C) we compare the prediction process of memorized idioms versus non-memorized idioms and short sequences from Wikipedia (\texttt{wiki}). Memorized predictions exhibit two characteristic phases: candidate promotion and confidence-boosting.}
    \label{fig:main_intro}
\end{figure*}

Transformer language models (LMs) memorize instances from their training data \citep{carlini2021extracting, zhang2021counterfactual}, and evidence is building that such memorization is an important precondition for their predictive abilities \citep{lee-etal-2022-deduplicating,feldman2020does,feldman2020neural,raunak-etal-2021-curious,Raunak2022FindingME}. Still, it is unknown when models decide to output memorized sequences, and how these sequences are being retrieved internally from memory. 
Current methods for analyzing memorization~\citep{feldman2020neural,zhang2021counterfactual,carlini2022quantifying} use definitions that are based on models performance, which changes between models and often also between training runs. Moreover, these methods study memorization behavior in terms of the model's ``black-box'' behavior rather than deriving a behavioral profile of memory recall itself.

Our first contributions are to provide a definition and construct a dataset that allows probing memorization recall in LMs.
We define a set of criteria for identifying memorized sequences that do not depend on model behavior:\footnote{Literature often purports to ``define memorization'', resulting in a multitude of technical definitions with subtle differences, although we would expect this concept to be consistent and intuitive. Thus, instead of explicitly defining ``memorization'', we will define \emph{sufficient criteria} for detecting memorized instances.}
sequences that have a single plausible completion that is independent of context and can be inferred only given the entire sequence.
We show that many idioms (e.g., ``play it by ear'') fulfill these conditions, allowing us to probe and analyze memorization behavior.
Furthermore, we construct a dataset of such English idioms, dubbed \idiomem{}, and release it publicly for the research community.

Next, to analyze memory recall behavior, we compare the construction process of predictions that involve memory recall with those that do not.
To this end, given a LM, we create two sets of memorized and non-memorized idioms from \idiomem{} (Fig.~\ref{fig:main_intro}, A). We then adopt a view of the transformer inference pass as a gradual refinement of the output probability distribution \citep{geva-etal-2021-transformer,anthropic}. Concretely, the token representation at any layer is interpreted as a ``hidden'' probability distribution over the output vocabulary \citep{geva2022transformer} (Fig.~\ref{fig:main_intro}, B). 
This interpretation allows tracking the prediction across layers in the evolving distribution.
We find a clear difference in model behavior between memorized and non-memorized predictions (Fig.~\ref{fig:main_intro}, C). This difference persists across different transformer architectures and sizes: retrieval from memory happens in two distinct phases, corresponding to distinct roles of the transformer parameters and layers: (1) \emph{candidate promotion} of memorized predictions' rank in the hidden distribution in the first layers, and (2) \emph{confidence boosting} where, in the last few layers, the prediction's probability grows substantially faster than before. This is unlike non-memorized predictions, where the two phases are less pronounced and often indistinct.
We further confirm these phases of memorized predictions through intervention in the network's FFN sublayers, which have been shown to play an important role in the prediction construction process \citep{geva2022transformer, mickus-etal-2022-dissect}. Concretely, zeroing out hidden FFN neurons in early layers deteriorate memory recall, while intervention in upper layers does not affect it.

Last, we show our findings extend to types of memory recall beyond idioms by applying our method to factual statements from the \lamauhn{} dataset \citep{Poerner2020EBERTEE} (e.g. \nl{The native language of Jean Marais is French}). For factual statements that were completed correctly by the LM, we observe the same two phases as in memorized idioms, further indicating their connection to memory recall.

To summarize, we construct a novel dataset of idioms, usable for probing LM memorization irrespective of the model architecture or training parameterization. We then design a probing methodology that extracts carefully-devised features of the internal inference procedure in transformers. By applying our methodology and using our new dataset, we discover a profile that characterizes memory recall, across transformer LMs and types of memorized instances.
Our released dataset, probing framework, and findings open the door for future work on transformer memorization, to ultimately demystify the internals of neural memory in LMs.

\section{Criteria for Detecting Memory Recall}
\label{sec:input_criteria}

To study memory recall, we require a set of inputs that trigger this process. Prior work on memorization focused on detecting instances whose inclusion in the training data has a specific influence on model behavior, such as increased accuracy on those instances \citep{feldman2020neural,magar-schwartz-2022-data, carlini2022quantifying, carlini2021extracting, carlini2019secret}. 
As a result, memorized instances differ across models and training parameterization.
Our goal is instead to find a stable dataset of sequences that correctly predicting their completion indicates memorization recall.
This will greatly reduce the overhead of studying memorization and facilitate useful comparisons across models and studies.

To build such a dataset, we start by defining a general set of criteria that are predicates on sequence features, entirely independent of the LM being probed.
Given a textual sequence of $n$ words, we call the first $n-1$ words the \textit{prompt} and the $n$th word the \textit{target}. We focus on the task of predicting the target given the prompt, i.e., predicting the last word in a sequence given its prefix.\footnote{In cases where tokenization divides the target to sub-tokens, our task becomes predicting the target's first token.}
Such predictions can be based on either generalization or memorization, and we are interested in isolating memorized cases to study model behavior on them. Particularly, we are looking for \emph{sequences for which success in this task implies memorization recall}.

We argue that the following criteria are \emph{sufficient} for detecting such memorized sequences:
\begin{enumerate}
[leftmargin=*,topsep=1pt,parsep=1pt]
    \item \textbf{Single target, independent of context:} We require that the target is the only correct continuation, regardless of the textual context where the prompt is placed.\footnote{We assume that contexts are naturally-occurring and not adversarial.}

    \item \textbf{Irreducible prompt:} The target is the single correct completion only if the entire prompt is given exactly. Changing or removing parts from the prompt would make the correct target non-unique.

\end{enumerate}

\begin{claim}
\label{claim:criteria}
Assume a sequence fulfills the above criteria. Then, if a LM correctly predicts the target, it is highly likely that this prediction involves memory recall.
\end{claim}

\begin{proof}[Justification]
First, observe that most natural-language prompts have many possible continuations. For example consider the sentence \nl{to get there fast, you can take this \textunderscore\textunderscore\textunderscore\textunderscore}. Likely continuations include \nl{route}, \nl{highway}, \nl{road}, \nl{train}, \nl{plane}, \nl{advice}, inter alia. Note that there are several divergent interpretations or contexts for the prompt, and for each, language offers many different ways to express similar meaning.

A prediction that is a product of generalization --- i.e., it is derived from context and knowledge of language --- \emph{always has plausible alternatives}, depending on the context and stylistic choice of words. Hence, the relationship between the entire prompt and the target, where the target is the single correct continuation, is something that needs to be memorized rather than derived via generalization.
A LM that predicts \emph{the} single correct continuation either memorized this relationship, or used ``cues'' from the prompt that happen to provide indication towards the correct continuation. To illustrate the latter, consider the sequence \nl{it's raining cats and \textunderscore\textunderscore\textunderscore\textunderscore} which has a single correct continuation, \nl{dogs}, but a LM might predict it without observing this sequence during training, due to the semantic proximity of \nl{cats} and \nl{dogs}. Our second criterion excludes such cases by requiring that the correct continuation is only likely given the entire sequence.

Therefore, a LM that correctly completes a sequence that fulfills both criteria, is likely to have recalled it from memory.
\end{proof}

In the next section, we argue that idioms are a special case of such sequences, and are thus useful for studying memorization (\S\ref{sec:probing_idioms}).

\section{The Utility of Idioms for Studying Memorization}
\label{sec:probing_idioms}

An idiom is a group of words with a meaning that is not deducible from the meanings of its individual words. For example, consider the phrase \nl{play it by ear} --- there is a disconnect between its nonsensical literal meaning (to play something by a human-body organ called `ear') and its intended idiomatic meaning (to improvise).

A key observation is that \emph{idioms often satisfy our criteria (\S\ref{sec:input_criteria}), and therefore can probe memorization}. 
First, by definition, idioms are expected to be non-compositional \cite{dankers-etal-2022-transformer}. 
They are special ``hard-coded'' phrases that carry a specific meaning. As a result, their prompts each have a single correct continuation, regardless of their context (criterion 1). For example, consider the prompt \nl{crying over spilt \textunderscore\textunderscore\textunderscore\textunderscore} --- a generalizing prediction would predict that this slot may be filled by any spillable item, like wine, water or juice, while a memorized prediction will retrieve only milk in this context. Notably, while this is an empirical characterization of many idioms, there might be exceptions, e.g., contexts that are adversarially chosen to change the completion.
Second, many idioms are ``irreducible'', for example, the sub-sequences \nl{crying over} or \nl{over spilt} by themselves have but a scant connection to the word \nl{milk}.

Still, not all idioms fulfill the criteria. For example, even when the idiom is far from literal, its constituents sometimes strongly indicate the correct continuation, such as with the case of \nl{it's raining cats and \textunderscore\textunderscore\textunderscore\textunderscore} (as explained in~\S\ref{sec:input_criteria}). To construct a dataset of memorization-probing sequences, we will carefully curate a set of English idioms and filter out ones that do not fulfill our criteria.

\subsection{The \idiomem{} Dataset}
\label{subsec:idiomem}

We begin with existing datasets of English idioms: \magpie{}~\citep{haagsma-etal-2020-magpie},\footnote{We take idioms with annotation confidence of {$>75\%$} and exclude frequently occurring literal interpretations.} \epie{} \citep{saxena2020epie}, and the English subset of \lidioms{} \citep{moussallem-etal-2018-lidioms}. We enrich this collection with idioms scraped from the website ``Education First'' (EF).\footnote{\url{https://www.ef.com/wwen/english-resources/english-idioms/}}
We then split each idiom into a prompt containing all but the last word and a target that is the last word. 
Next, we filter out idioms that do not comply with our criteria (\S\ref{sec:input_criteria}) or whose target can be predicted from their prompt based on spurious correlations rather than memorization.
To this end, we use three simple rules:

\begin{itemize}
[leftmargin=*, topsep=1pt, parsep=1pt]

\item \textbf{Short idioms}. We observe that prompts of idioms with just a few words often have multiple plausible continuations, that are not necessarily the idiom's target, violating our first criterion. 
For example, the prompt \nl{break a \textunderscore\textunderscore\textunderscore\textunderscore} has many possible continuations (e.g. \nl{window}, \nl{promise}, and \nl{heart}) in addition to its idiomatic continuation \nl{leg}. To exclude such cases, we filter out idioms with $<4$ words.

\begin{table}[t]
    \setlength{\belowcaptionskip}{-10pt}
    \centering
    \footnotesize
    \begin{tabular}{@{}lcc@{}}
    \textbf{Source} & \textbf{\# of Idioms} & \textbf{Idiom Length} (words)
    \\ \toprule
    \magpie{} & 590 & $4.5\pm0.9$ \\ 
    \lidioms{} & 149 & $5.1\pm1.2$ \\ 
    \efdata{} & 97 & $5.6\pm1.9$ \\ 
    \epie{} & 76 & $4.4\pm0.7$ \\ 
    \midrule
    \textbf{Total} (unique) & 814 & $4.7 \pm 1.8$ 
    \end{tabular}
    \caption{Statistics per data source in \idiomem{}.
    }
    \label{table:idiom_stats}
\end{table}

\item \textbf{Idioms whose target is commonly predicted from the prompt's subsequence.}
We filter such cases to ensure the prompt fulfills our second criterion (prompt irreducibility).

To detect these cases, we use an ensemble of pretrained LMs: \gptm{}, \textsc{RoBERTa-base} \citep{Liu2019RoBERTaAR}, \textsc{T5-base} \citep{kale-rastogi-2020-text} and \textsc{ELECTRA-base-generator} \citep{Clark2020ELECTRA}, and check for each model if there is an n-gram ($1\leq n\leq4$) in the prompt from which the model predicts the target. We filter out idioms for which a majority ($\geq 3$) of models predicted the target (for some n-gram). 

\item \textbf{Idioms whose targets are semantically similar to tokens in the prompt.} To further ensure prompt irreducibility, we embed the prompt's tokens and the target token using GloVe word embeddings \citep{pennington-etal-2014-glove}. We measure the cosine distance between the target token to each token in the prompt separately and take the maximum of all the tokens. We filter out idioms where this number is higher than 0.75 (this number was tuned manually using a small validation set of idioms).

\end{itemize}

Overall, 55.7\% of the idioms were filtered out, including 48.5\% by length, 6.1\% by the predictable-target test, and an additional 1.6\% by the prompt-target similarity, resulting in an 814 idioms dataset, named \idiomem{}.
Further statistics are provided in Tab.~\ref{table:idiom_stats}, and example idioms in Tab.~\ref{table:examples}.

\begin{table}[t]
\setlength\tabcolsep{4.0pt}
\setlength{\belowcaptionskip}{-8pt}
\centering
\footnotesize
\begin{tabular}{@{}p{2.2cm}l|cc|c@{}}
\textbf{Prompt} & \textbf{Target} & \textbf{Pred.} & \textbf{Sim.} & 
\textbf{\idiomem{}} \\
\toprule
\nl{make a mountain out of a} & \texttt{molehill} &  & & \checkmark \\
\nl{think outside the} & \texttt{box} &  & &   \checkmark\\
\nl{there's no such thing as a free} & \texttt{lunch} &  & & \checkmark \\ \midrule
\nl{go back to the drawing} & \texttt{board} & \checkmark &   & \\ 
\nl{boys will be} & \texttt{boys}  &  & \checkmark  &  \\ 
\nl{take it or leave} & \texttt{it} & \checkmark &  \checkmark &  \\
\bottomrule
\end{tabular}
\caption{Example English idioms included and excluded from \idiomem{} by our filters of predictable target (Pred.) and prompt-target similarity (Sim.).}
\label{table:examples}
\end{table}

\begin{table*}[t]
    \setlength{\belowcaptionskip}{-8pt}
    \centering
    \footnotesize
    \begin{tabular}{@{}lcccc@{}}
    &  \textbf{\gptm{}} & \textbf{\gptl{}} & \textbf{\bertb{}} & \textbf{\bertl{}} \\
    \midrule
    Memorized idioms (\mem{}) & 364 \tcbox{\footnotesize{44.7\%}} & 392 \tcbox{\footnotesize{48.2\%}} &  230 \tcbox{\footnotesize{28.3\%}}  & 305 \tcbox{\footnotesize{37.5\%}} \\
    Non-memorized idioms (\unmem{}) & 450 \tcbox{\footnotesize{55.3\%}} & 422 \tcbox{\footnotesize{51.8\%}} &  584 \tcbox{\footnotesize{71.7\%}} & 509 \tcbox{\footnotesize{62.5\%}} \\
    \bottomrule
    \end{tabular}
    \caption{Number of memorized idioms vs. non-memorized idioms from the \idiomem{} dataset for each model. An instance is considered a memorized example if the model correctly predicts the target.}
    \label{tab:models_sets}
\end{table*}

\section{Probing Methodology}
\label{sec:interpretation}

\paragraph{Background and Notation}
Assuming a transformer LM with $L$ layers, a hidden dimension $d$ and an input/output-embedding matrix $E\in\R^{\lvert V \rvert \times d}$ over a vocabulary $V$. Denote by $\mathbf{s} = \langle s_1,...,s_t \rangle$ the input sequence to the LM, and let $\mathbf{h}_i^\ell$ be the output for token $i$ at layer $\ell$, for all $\ell\in 1,...,L$ and $i\in 1,...,t$. The model's prediction for a token $s_i$ is obtained by projecting its last hidden representation $\mathbf{h}_i^L$ to the embedding matrix, i.e. $softmax(E\mathbf{h}_i^L)$.

Following \cite{geva-etal-2021-transformer,geva2022transformer}, we interpret the prediction for a token $s_i$ by viewing its corresponding sequence of hidden representations $\mathbf{h}_i^1, ..., \mathbf{h}_i^L$ as an evolving distribution over the vocabulary. Concretely, we read the ``hidden''  distribution at layer $\ell$ by applying the same projection to the hidden representation at that layer: $\mathbf{p}_i^\ell = softmax(E\mathbf{h}_i^\ell)$.
Using this interpretation, we track the probability and rank of the predicted token in the output distribution across layers. A token's rank is its position in the output distribution when sorted by probability from highest to lowest (e.g. the rank of the final predicted token is zero).

\paragraph{Probing Procedure}
Our key method to understand how transformer LMs retrieve information from memory is \emph{comparing features of memory recall} to inference that does not necessarily include memory recall. Given a set of sequences that fulfill the criteria in~\S\ref{sec:input_criteria}, we split them into a ``memorized'' set whose targets' first token is predicted correctly by the model being analyzed given (and are therefore memorized), and a ``non-memorized'' set whose target is predicted incorrectly.
We additionally include a second set of ``non-memorized'' instances: natural-language sequences randomly sampled from a large corpus (we assume most naturally-occurring sequences are not memorized).

To probe a LM, we run it on the 3 sets, and for each set and each layer, we (a) extract the rank and probability of the final predicted token in the hidden distribution for each prompt, and (b) compute the mean rank and probability over all prompts.

\section{Probing Memorization Using Idioms}
\subsection{Experimental Setup}
\label{sec:experimental}

\paragraph{Datasets}
For each LM under analysis (see below), we split \idiomem{} into two disjoint subsets of memorized and non-memorized idioms, denoted as \mem{} and \unmem{}, respectively, according to whether or not the LM succeeds in completing them.  We produce an additional set of non-memorized instances, \wiki{}, by sampling prompts from the \wikitext{} dataset \citep{merity2017pointer} of the same length distribution as in \idiomem{} (see Tab.~\ref{table:idiom_stats}).

\paragraph{Models} 
We use multiple transformer LMs that are different in size, architecture, and optimization objective. We use \textsc{GPT2} (medium and large) \citep{Radford2019LanguageMA}, an autoregressive transformer decoder, and \textsc{BERT} (base and large) \citep{devlin-etal-2019-bert}, a transformer encoder trained with a masked language modeling (MLM) objective. To evaluate \textsc{BERT} on a specific idiom, we feed the idiom's prompt concatenated with the special mask token and a period (e.g. \nl{think outside the \texttt{[MASK]}.}).
Further details on each model are presented in Tab.~\ref{tab:models_hyperparameters}.
The number of memorized and non-memorized idioms from \idiomem{} for each model are provided in Tab.~\ref{tab:models_sets}.

\begin{figure*}[t]
    \setlength{\belowcaptionskip}{-8pt}
    \begin{center}
    \includegraphics[scale=0.28]{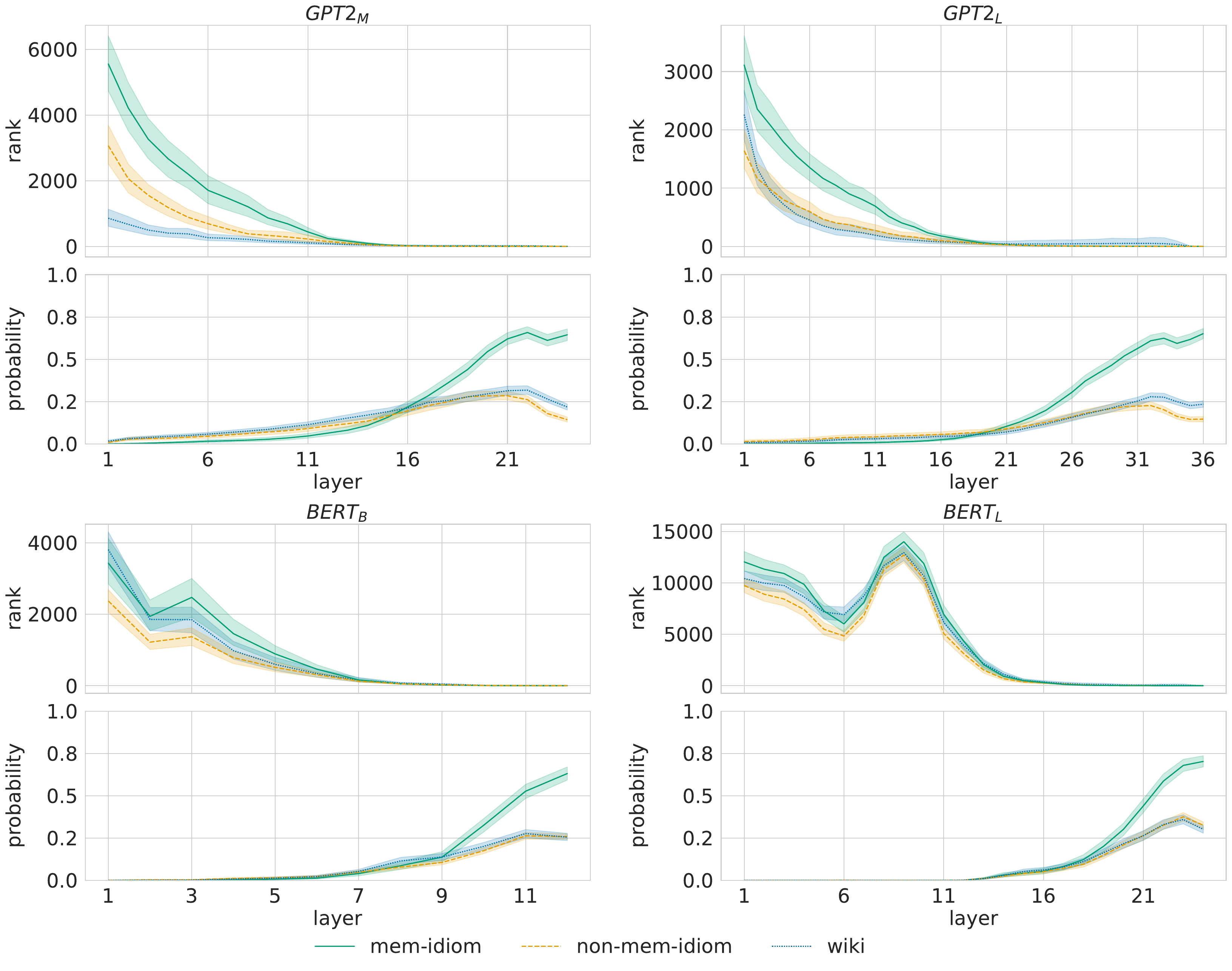}
    \end{center}
    \caption{The predicted token's probability and rank across layers, for memorized idioms (\mem{}), non-memorized idioms (\unmem{}) and short sequences from Wikipedia (\texttt{wiki}). Memory recall exhibits two characteristic phases of candidate promotion and confidence boosting.
    }
    \label{figure:intro}
\end{figure*}

\subsection{Memorized Predictions are a Two-Step Process}
\label{sec:two_phases}

Fig.~\ref{figure:intro} shows the probability and rank of the output token across layers, for memorized and non-memorized idioms and short sequences from Wikipedia, by \gptm{}, \gptl{}, \bertb{}, and \bertl{}.  Naturally, the prediction's rank decreases across layers as the prediction probability increases. However, for memorized predictions these trends occur as two distinct and sharp inference phases.
In lower layers, the prediction's rank decreases from a high rank to near zero, while its probability is also close to zero. For example, in \gptl{} the rank decreases until layer 20 while the probability remains below 0.1. We refer to this phase as \emph{candidate promotion}, as the predicted token is being promoted to be a top candidate in the output distribution. 

Compared to non-memorized predictions, the initial rank of memorized predictions is generally higher, especially in \gpt{} (6000 vs. 3000 in \gptm{}, and 3000 vs. 1500 in \gptl{}). 
A potential explanation would be a generally lower frequency of the predicted token for memorized idioms. However, we find there is only low negative correlation between the initial rank of the predicted token and its frequency (see Tab.~\ref{tab:rank_freq_corr}).
We therefore offer a different explanation: non-memorized predictions are often promoted in early layers that detect local ``shallow'' patterns, such as common bigrams~\cite{geva-etal-2021-transformer}, while predictions for memorized idioms are not local as they requires processing of the entire input.

In the middle layers, once the predicted token reaches the top of the hidden distribution, its probability increases until the last layer. We refer to this phase as \emph{confidence boosting}, as the distribution shifts towards the predicted token. For memorized idioms, this increase is abrupt and dramatic, with a final probability of $>0.6$ across all models. In comparison, predictions on short sequences from Wikipedia and non-memorized idioms have a substantially lower probability of $\sim0.2$. This can be explained by the fact that memorized idioms have a single correct target, rather than many possible continuations, as in the instances from Wikipedia. In addition, low-probability predictions for non-memorized idioms are expected as the model did not memorize the idioms and does not know their continuation. In~\S\ref{apx:two_step_add}, we provide more fine-grained analysis of these trends via a log-scaled view of the prediction's rank and a visualization of the ranks and probabilities for separate clusters of the memorized predictions.

We further verify that our extracted hidden-distribution features distinguish between memorized and non-memorized predictions by training linear classifiers over combinations of these features (details in \S\ref{apx:classifiers}). We observe that, indeed, our features enable separation between memorized and non-memorized predictions at high accuracy (77\%-85\% across models). Moreover, classifiers that use hidden distribution  features are more accurate than those relying only on the model's output. Overall, these findings provide a profile of memorized predictions, suggesting that the memorized information is retrieved in early layers at inference.

\begin{table}[t!]
    \centering
    \small
    \begin{tabular}{lcl}
          & correlation & p-value  \\
         \midrule
         \gptm{} & -0.22 & $2.9e^{-21}$  \\
         \gptl{} & -0.18 & $5.4e^{-14}$ \\
         \bertb{} & -0.19 & $1.8e^{-15}$ \\
         \bertl{} & 0.15 & $3.5e^{-10}$ \\
         \bottomrule
    \end{tabular}
    \caption{Pearson correlation between the predicted token's rank at the first layer and its general frequency in Wikipedia.}
    \label{tab:rank_freq_corr}
\end{table}

\begin{figure*}[t]
    \setlength{\belowcaptionskip}{-10pt}
    \begin{center}
    \includegraphics[scale=0.37]{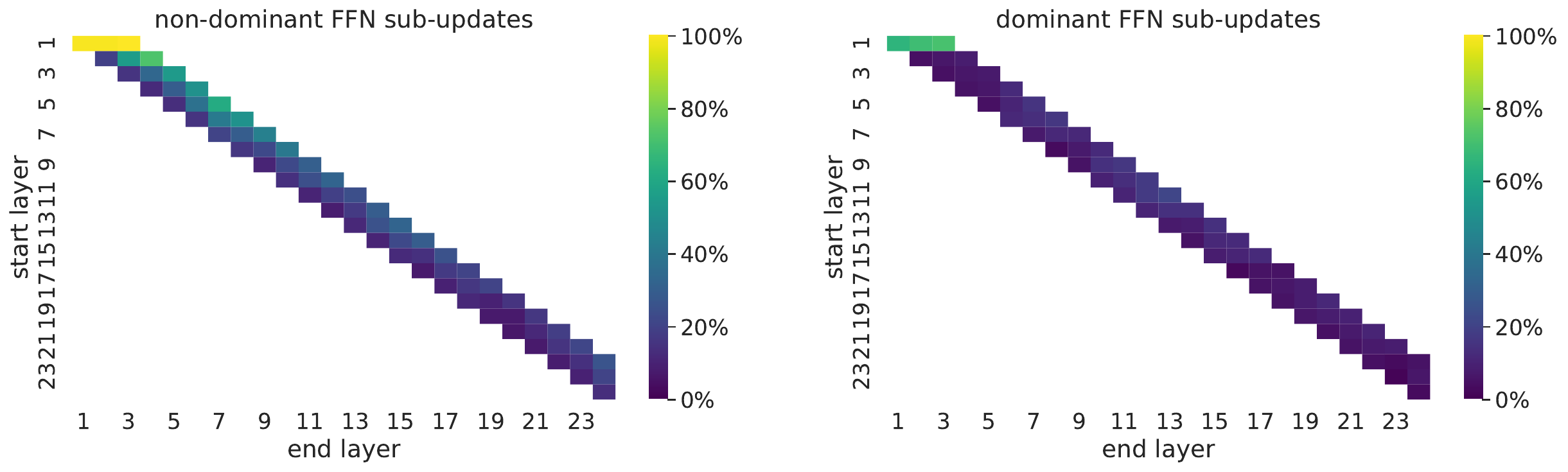}
    \end{center}
    \caption{
    Intervention in
    non-dominant (left) and dominant (right) FFN sub-updates in \gptm{}. Each cell shows the percentage of memorized idioms for which the prediction was changed by zeroing-out the FFN sub-updates between the start and end layers.}
    \label{figure:knockout}
\end{figure*}

\subsection{Testing the Roles of Different Layers Through Intervention} 
\label{sec:knockouts}

Our analysis in the previous section interprets hidden representations as distributions over the output vocabulary. 
We now conduct intervention experiments to verify that this interpretation is meaningful for studying memory recall, and to test layers' roles in that process.
Concretely, we zero out FFN sub-updates to the hidden representation (\S\ref{sec:interpretation}) and measure changes in memorized predictions.

\paragraph{A Short Primer on Transformer FFN Sublayers.}
FFN sublayers are the final computation in transformer layers, which output the hidden distribution at the center of our analysis. In general, they have a key role in capturing knowledge in transformer LMs \citep{dai-etal-2022-knowledge, meng2022locating}.
We follow \citet{geva2022transformer} and view the computation of each FFN sublayer as a weighted collection of $d_m$ sub-updates to the output distribution, each promoting a concept in the vocabulary space, e.g. ``past-tense verbs'' or ``female athletes''. To understand this, consider the computation of the FFN at layer $\ell$, given by $\texttt{FFN}^\ell(\mathbf{h}^\ell_i) = f(W_K^\ell \mathbf{h}_i^\ell)W_V^\ell$, where $W_K^\ell, W_V^\ell \in \R^{d_m \times d}$ are parameter matrices, $d_m$ is the intermediate hidden dimension, and $f$ is a point-wise non-linearity activation function. This computation can be decomposed as: $\texttt{FFN}^\ell(\mathbf{h}^\ell_i) = \sum_{j=1}^{d_m} f(\mathbf{h}_i^\ell\cdot \mathbf{k}_j^\ell)\mathbf{v}_j^\ell =  \sum_{j=1}^{d_m}m_j^\ell \mathbf{v}_j^\ell$, where $\mathbf{k}_j^\ell$ and $\mathbf{v}_j^\ell$ are the $j$-th row in $W_K^\ell$ and the $j$-th column in $W_V^\ell$, respectively.
Geva et al. argue that each weight $m_j^\ell$ is the score assigned by the model for some textual pattern, and each vector $\mathbf{v}_j$ promotes a concept that follows that pattern.

\paragraph{Experiment}
First, we sample 100 random instances from \idiomem{} that the model memorized. Then, for each range of up to 3 consecutive layers, we perform two complementary experiments, where we run \gptm's inference on the 100 instances while intervening in the chosen layers to cancel the contribution of FFN sub-updates to the prediction. Specifically, for each layer $\ell$ in the layer range, we perform the following two interventions: first, we zero out (i.e. artificially set to 0 during inference) the weights of the 10 most dominant sub-updates, which are known to be particularly salient for predictions~\cite{geva2022transformer} (there are $d_m$ sub-updates per layer). Concretely, we sort the sub-updates by their contribution to the FFN output $|m_i^\ell| ||\mathbf{v}_i^\ell|| \;\forall i\in [1,...,d_m]$, and set $m_i^\ell = 0$ for the 10 sub-updates with the highest contribution. 
Next, we zero out non-dominant sub-updates, i.e. all the sub-updates \emph{except} for the 10 most dominant ones.
For each intervention, we measure how often it changes the predicted token. Further measurements of the change in rank and probability of the target token are reported in \S\ref{apx:knockout_add}.

\paragraph{Results}
Fig.~\ref{figure:knockout} shows, for each layer range, the percentage of memorized idioms where the predicted token has changed.
Focusing on zeroing out non-dominant FFN sub-updates (Fig.~\ref{figure:knockout}, left), we observe a \emph{two-phase pattern of decreasing ``layer importance'' which corresponds to the two-phase pattern of decreasing rank and increasing probability during inference (\S\ref{sec:two_phases})}: layers' effect on memory recall drops precipitously in the first 10 layers, corresponding to the candidate promotion phase. Then, from around layer 10 onwards, the drop in effect is much less steep, corresponding to the confidence-boosting phase. Intervention in upper layers rarely changes the predicted token, and its effect is limited to reducing the model's confidence (\S\ref{apx:knockout_add}). We further visualize this two-phase behavior in Fig.~\ref{fig:intervention_triplets} in \S\ref{apx:knockout_add}.

These findings suggest that the candidate promotion phase, while having a smaller effect on the prediction's assigned probability compared to the later confidence-boosting stage, in fact, has a crucial role in memory recall.

We also observe that interventions in the first layer are by far the most destructive, with 100\% change in prediction for non-dominant updates (Fig.~\ref{figure:knockout}, left).
Unlike for other layers, this is also observable when zeroing out the dominant sub-updates (Fig.~\ref{figure:knockout}, right), which constitute only 0.1\% of layer sub-updates. This suggests the first layer is especially critical for memorized predictions.

\section{Memorization of Factual Statements}
\label{sec:facts}

We now examine if our findings generalize beyond idioms to other types of memory recall, focusing on the completion of statements expressing facts.

\paragraph{Data}
Datasets for evaluating memorization of factual knowledge typically contain simple queries such as \nl{The continent of Kuwait is}, where predicting the next token requires knowledge of the triplet $\langle s, r, t \rangle$ where $s$ is a source entity (e.g., \texttt{Kuwait}), $t$ a target entity (e.g., \texttt{Asia}), and $r$ is the relation between them (e.g., \texttt{is the continent of}). 
However, unlike idioms, such queries are not suitable for probing memorization since they often fail to satisfy the criteria in \S\ref{sec:input_criteria}. Concretely, queries often include ``clues'' that could make the prediction easy to guess and based on generalization \citep{Poerner2020EBERTEE} (e.g. predicting a Spanish-speaking country for the query \nl{Federico López was born in}), and the same fact can be expressed in multiple different ways (e.g. \nl{Kuwait is a country in Western Asia} also encodes the above fact).

To test memorization of facts, we, therefore, collect factual statements where such cases are excluded. We use \lamauhn{} \citep{Poerner2020EBERTEE}, a subset of the \lama{} dataset \citep{petroni-etal-2019-language}, where ``easy-to-guess'' queries are filtered out. \lama{} comprises of queries structured as ``fill-in-the-blank'' cloze statements (e.g. \nl{Gordon Scholes is a member of the \textunderscore\textunderscore\textunderscore\textunderscore political party.}). To accommodate autoregressive LMs, we consider only queries where the blank appears at the end. 
In addition, we keep only queries with a single correct completion (based on \textsc{LAMA}). 
This turns our definition of memorized and non-memorized sets (\S\ref{sec:interpretation}) equivalent to separating based on the evaluation metric of LAMA:
an instance is considered memorized if the model predicts the single correct completion and non-memorized otherwise. Overall, the resulting collection consists of 17,855 factual statements with 22 unique relations.

\paragraph{Memorized Facts Exhibit a Similar Prediction Process to Idioms}
\label{subsec:resultsprocess}

We repeat our analysis (\S\ref{sec:experimental}), using the collected factual statements and \gptm{}. Splitting the statements to those memorized and non-memorized by the model results in 786 \texttt{mem-facts} vs. 17,069 \texttt{non-mem-facts} statements.

Fig.~\ref{figure:facts_main} shows the rank and probability of the predicted token across layers. Like idioms, memorized facts exhibit a clear two-phase prediction process, where the prediction probability rapidly increases once the candidate reaches a low rank (at layer 16). This is in contrast to non-memorized facts and short sequences from Wikipedia, where the rank (probability) gradually decreases (increases) across layers without a distinct two phases.

\begin{figure}[t]
    \setlength{\belowcaptionskip}{-10pt}
    \begin{center}
    \includegraphics[scale=0.39]{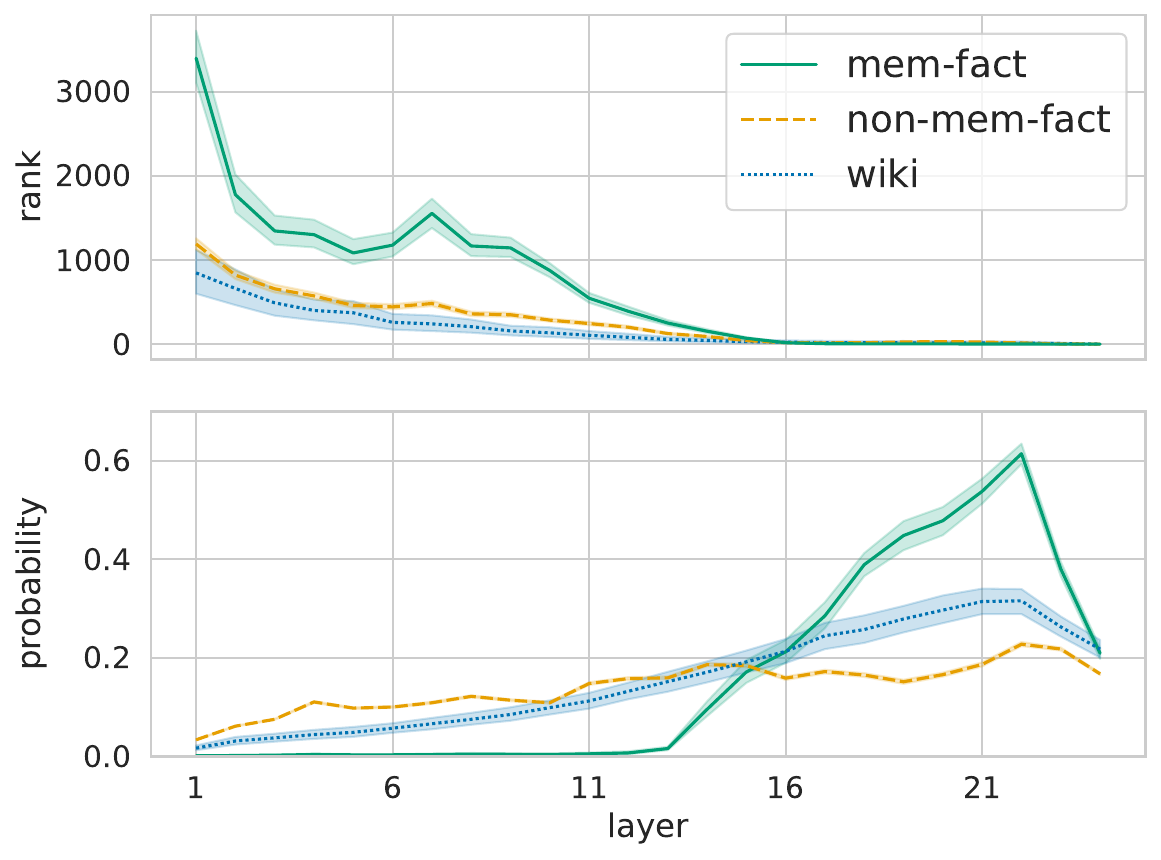}
    \end{center}
    \caption{The predicted token's probability and rank across layers of \gptm{}, for memorized (\texttt{mem-fact}) and non-memorized (\texttt{non-mem-fact}) facts and short sequences from Wikipedia (\texttt{wiki}).
    }
    \label{figure:facts_main}
\end{figure}

\paragraph{Differences from Memorized Idioms Stem from Ill-Defined Targets}
There is one major difference compared to memorized idioms (Fig.~\ref{figure:facts_main} vs. Fig.~\ref{figure:intro} upper left), which is a substantial drop in probability ($0.62 \rightarrow 0.21$) in the last two layers. We hypothesize that this is because the input query has multiple plausible completions that were not specified as ``correct'' targets in \textsc{LAMA}.
We verify this by manually analyzing predictions, and find that for 82 out of 100 queries, there is more than one correct continuation in the top 5 predicted tokens. 
We posit that the above deficiency of \textsc{LAMA} is inherent because, in violation of our criterion (\S\ref{sec:input_criteria}), factual statements can usually be expressed in many ways so their prompt has no single correct target.

\section{Related Work}

\paragraph{Memorization as Training-Data Influence.}
Memorization in LMs has attracted immense attention due to their rich sensitive training data~\citep{carlini2019secret, song2019auditing, carlini2021extracting, zhang2021counterfactual, carlini2022quantifying, tirumala2022memorization, tanzer-etal-2022-memorisation,Raunak2022FindingME}.
Recent work suggests that memorization is \emph{necessary} for performant ML due to the ``long tail'' of infrequently-observed patterns~\citep{zhang2021understanding,feldman2020does,brown2021memorization,raunak-etal-2021-curious}. 
This line of work has two key limitations: (a) only black-box behavior is measured rather than looking at models' internal prediction process, and (b) it detects memorized instances by measuring the effect of their inclusion in the training set on inference behavior. This results in a set of memorized examples that is specific to the model, training data, and even training pass, making it difficult to build on these results in future research.

\paragraph{Transformers and Idioms.}
\citet{nedumpozhimana2021finding} showed that idioms are identified using textual cues within the expression; \citet{dankers-etal-2022-transformer} showed that idioms tend to be internally processed as single units of meaning.
It has also been known \citep{fakharian2021contextualized, salton2016idiom} that LM contextual embeddings encode information about whether or not an expression is idiomatic (vs. literal). \citep{shwartz-dagan-2019-still,chakrabarty-etal-2022-rocket} also studied representations and interpretation of non-compositional sequences, such as idioms. No prior work used idioms to probe LM memorization, which is one of the main contributions of this work. 
We release our dataset, \idiomem{}, to facilitate future research on memorization recall in LMs. Diagnostic datasets, such as \idiomem{}, have often proven useful in the past \cite{sugawara-etal-2022-makes, parrish-etal-2021-nope}.


\paragraph{Memorization of Factual Knowledge.}
An extensive line of work \citep{petroni-etal-2019-language, jiang2020can, poerner2019bert, lewis2020retrieval, elazar2021measuring} studied LMs' capacity to acquire relational knowledge during training.
Some attention has also been given to understanding the inner workings of factual-memory recall: \citet{singh-etal-2020-bertnesia} showed that some facts are retrieved from the bottom and intermediate layers, and \citet{meng2022locating} localized factual-knowledge recall within the feed-forward-component computation. Since factual statements do not fulfill our criteria, it is difficult to convincingly argue that correct predictions indicate memory recall, making it impossible to use them to isolate the effect of memorization.

\section{Conclusion}
We introduce a methodological framework for detecting and analyzing memorized predictions in transformer LMs.
This includes a set of criteria on textual sequences for probing memorized predictions, the \idiomem{} dataset of idioms fulfilling these criteria, and an interpretation method of prediction internals.
We characterize a behavioral profile that is unique to predictions involving memory recall and is observable across different LMs and data types. 
By providing these fundamental tools and initiating a thread of research on the phenomena we observe, we hope to empower future work toward demystifying transformer memorization.

\section*{Limitations}
Our criteria for detecting memorized instances are sufficient but not necessary, which raises the question of what other sequences that trigger memory recall satisfy them.

Additionally, while correct prediction for sequences that fulfill our criteria implies memory recall, incorrect prediction does not necessarily imply that no memory was used. This means that our set \unmem{} might include some portion of memorized sequences. This does not qualitatively affect our results as long as \mem{} still contains \emph{more} memorized predictions than \unmem, as is evidenced by the stark differences we observe between LMs' internal behaviors on these sets. Our work focuses on showing the utility of idioms for probing memorization, and opening up a new thread of research in this vein. We therefore leave further investigation of the above gaps for future work.

\section*{Acknowledgements}
We thank Gabriel Stanovsky for helpful feedback.
This project has received funding and resources from CIFAR through the Canada CIFAR AI Chair program, the Province of Ontario, companies sponsoring the Vector Institute, as well as the Computer Science Scholarship granted by the Séphora Berrebi Foundation, and the European Research Council (ERC) under the European Union's Horizon 2020 research and innovation programme, grant agreement No. 802774 (iEXTRACT).

\bibliography{anthology,custom}
\bibliographystyle{acl_natbib}

\clearpage
\appendix
\begin{table*}[t]
    \centering
    \small
    \begin{tabular}{l|cccc}
          & \gptm{} & \gptl{} & \bertb{} & \bertl{}  \\
         \midrule
          probability & $\bm{84.6 \pm 2.8} $ & $\bm{84.0 \pm 2.3}$ & $\bm{84.2 \pm 1.8}$ & $\bm{77.6 \pm 3.3}$ \\
          ranks & $63.2 \pm 1.6$ & $59.5 \pm 2.7$ & $73.3 \pm 3.3$ & $64.5 \pm 3.6$ \\
          probability + ranks & $83.6 \pm 2.7$ & $82.9 \pm 2.9$ & $82.7 \pm 2.2$ & $76.9 \pm 2.9$ \\
          ranks layer 1-12 + probability & $\bm{84.3 \pm 2.8}$ & $83.4 \pm 2.7$ & $82.7 \pm 2.2$ & $\bm{77.2 \pm 4.2} $ \\ \midrule
          probability last layer & $79.5 \pm 3.5$ & $83.3 \pm 2.3$ & $77.7 \pm 2.3$ & $74.2 \pm 4.8$ \\
          final hidden state & $72 \pm 2.1 $ &  $72.5 \pm 3.7 $ & $72.7 \pm 2.9$ & $66.0 \pm 3.0$ \\
          token ids & $58 \pm 0.2 $ & $59.4 \pm 0.2 $ & $58.8 \pm 0.2 $ & $71.6 \pm 0.2 $ \\
          random & $49.5 \pm 3.8$ & $49.7 \pm 4.7$ & $50.4 \pm 4.9$ & $50.1 \pm 3.5$ \\
         \bottomrule
    \end{tabular}
    \caption{Cross-validation accuracy of a logistic-regression classifier trained to distinguish between memorized and non-memorized idioms.}
    \label{tab:classifiers}
\end{table*}

\section{Distinguishing Memorization Using Hidden-Distribution Features}
\label{apx:classifiers}
\S\ref{sec:two_phases} shows differences in our extracted hidden-distribution features, namely the rank and probability across layers, between memorized idioms and non-memorized sequences. To verify that these features are distinguishing between memorized and non-memorized predictions, we build a classifier that receives them as input, as follows.

\paragraph{Experiment}
To answer the above, we represent every instance in \idiomem{}, for every LM we experiment with, as a sequence of probabilities and ranks assigned to the predicted token at each layer. This results in a feature vector for each instance in \idiomem{} for each of our LMs (\gptm{}, \gptl{}, \bertb{}, and \bertl{}). We then append a class label for each LM's instances corresponding to whether it memorized them. Then, for each LM's dataset, we perform 10-fold cross-validation with an 80\%-20\% train-test split to evaluate the accuracy of a logistic-regression classifier using the LogisticRegresion classifier of scikit-learn,\footnote{\url{https://scikit-learn.org/}} specifying L1 penalty and a bilinear solver as constructor parameters. After each split and before evaluating the classifier, we balance the test set by replacing the larger class with a random subsample the size of the smaller class. To isolate the distinguishing utility of ranks from that of probabilities, we repeat this process while only taking either of them as features at each time. We also repeat this process while separately considering just the last-layer probability as a single feature, the last hidden state vector, and finally, the ranks in layers of layers 1-12 (omitting ranks in layers 13-16 where ranks are usually 0) appended to all layer probabilities.

\paragraph{Results}
Results are given in Tab.~\ref{tab:classifiers}. We observe that, across all models, most classifiers perform well over the 50\%-accuracy baseline for their class-balanced test sets. The output probability alone is often highly distinguishing with around 78\% accuracy, but the vector of 16 hidden-distribution probabilities seems to contain additional distinguishing information, as using it alongside ranks results in higher accuracy, typically around 84\%. Using the ranks in addition to probabilities usually \emph{decreases} accuracy, but omitting the ranks in layers 13-16 (which we know are mostly 0 as this is the confidence-boosting phase) attenuates this effect. We conjecture that ranks have little meaningful information, especially in the confidence-boosting phase.

\begin{figure}[t]
    \setlength{\belowcaptionskip}{-10pt}
    \begin{center}
    \includegraphics[scale=0.29]{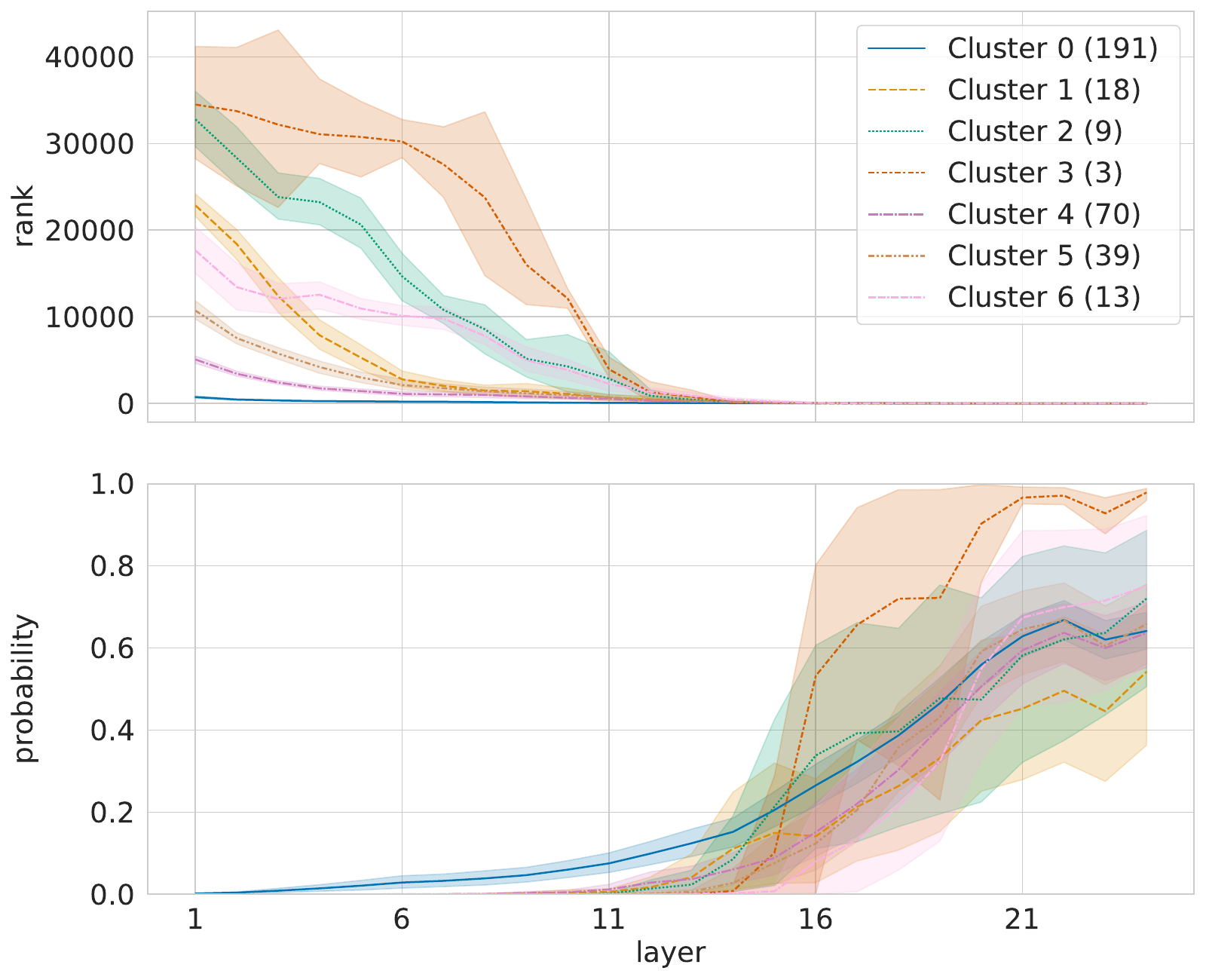}
    \end{center}
    \caption{The predicted token's probability and rank across layers for each cluster, after clustering the memorized idioms in \gptm{} according to rank and probability across layers.}
    \label{figure:clusters}
\end{figure}

\section{Fine-grained Analysis of Memorized Predictions}
\label{apx:two_step_add}
\S\ref{sec:two_phases} shows how recall of memorized predictions exhibits two characteristic phases (specifically, see Fig.~\ref{figure:intro}). To shed light on this phenomenon, we conduct additional analysis.

\subsection{Memorized Idioms In-depth Breakdown}
\label{apx:clusters}
In our analysis, we address all memorized predictions jointly. We now check whether these averaged results are consistent across subgroups of the memorized idioms.
To this end, we cluster the memorized idioms by \gptm{}, using the same hidden features as in \S\ref{apx:classifiers}, i.e., each instance is represented by the predicted token probability and rank across layers. We then cluster the idioms into seven groups, using K-mean clustering, and visualize the prediction rank and probability across layers for each group.
We set the number of clusters to $k=7$ based on manual inspection, and as we observed no substantial differences in the resulting clusters for larger values of $k$.

Results are shown in Fig.~\ref{figure:clusters}. We find that all groups exhibit the confidence-boosting phase, as the prediction probability quickly increases starting from the intermediate layers. Notably, the lowest final probability observed is $>0.5$, which is substantially higher than the average probability of $\sim0.2$ for non-memorized predictions (\S\ref{sec:two_phases}).  
However, considering the prediction rank for the different groups, we observe a relatively large variation.
Specifically, we observe that 55\% of the instances (cluster 0) have a low initial rank. This further supports the findings in \S\ref{apx:classifiers}.

\subsection{Log-Scale Visualization}
\label{apx:log_scale}
Fig.~\ref{figure:2step_log} shows a log-scaled view of the graphs from Fig.~\ref{figure:intro}. We observe that, in terms of orders of magnitude (vs. absolute value), the differences in initial ranks between memorized and non-memorized predictions are more minor, especially in BERT models, whereas rank differences measured in upper layers are more stark.

\begin{figure*}[t]
    \setlength{\belowcaptionskip}{-8pt}
    \begin{center}
    \includegraphics[scale=0.28]{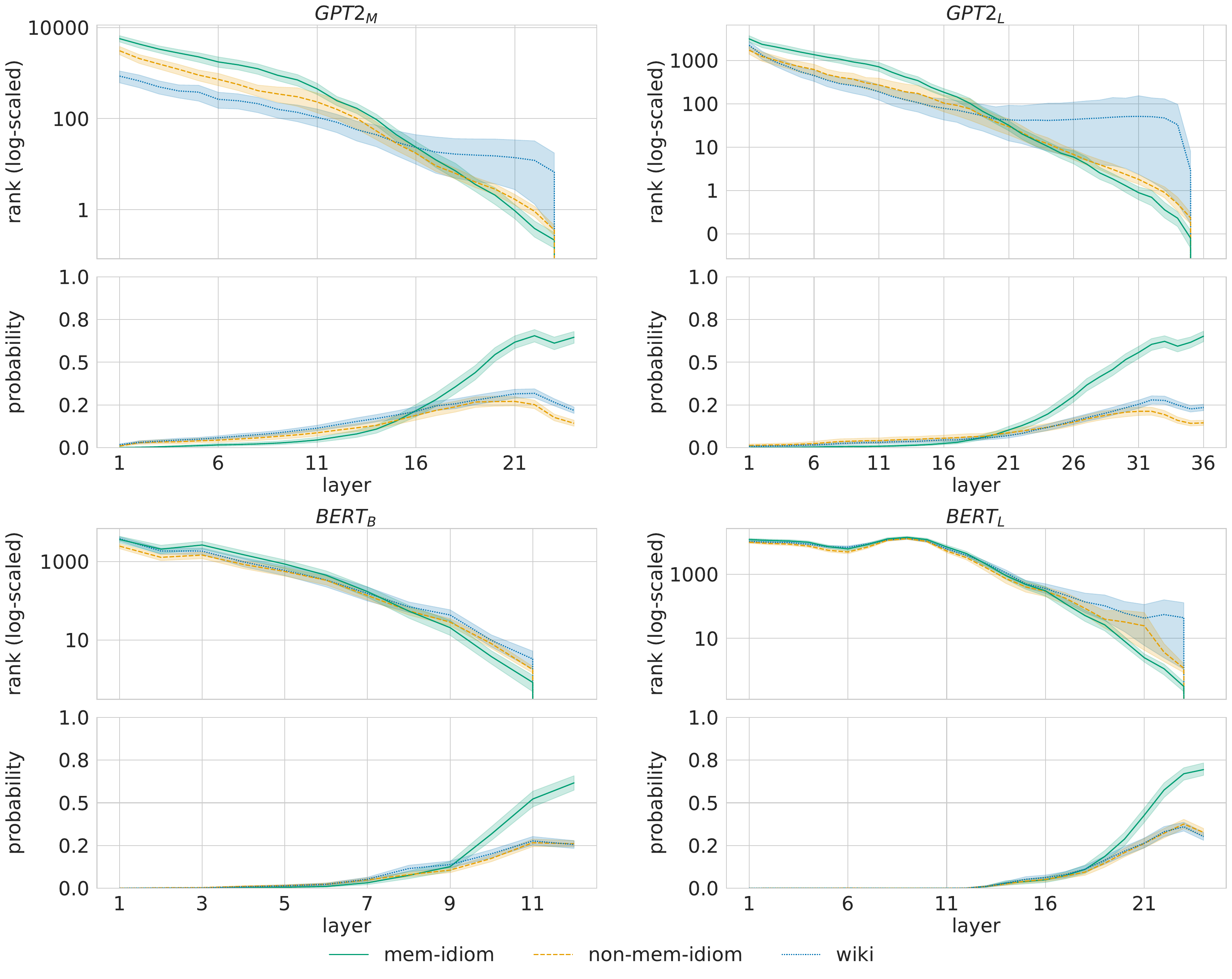}
    \end{center}
    \caption{The predicted token's probability and log-scaled rank across layers, for memorized idioms (\mem{}), non-memorized idioms (\unmem{}) and short sequences from Wikipedia (\texttt{wiki}).
    }
    \label{figure:2step_log}
\end{figure*}

\clearpage

\section{Intervention Experiments: Additional Analysis}
\label{apx:knockout_add}

\subsection{Additional Visualization}

\begin{figure}[t]
    \centering
    \includegraphics[scale=0.34]{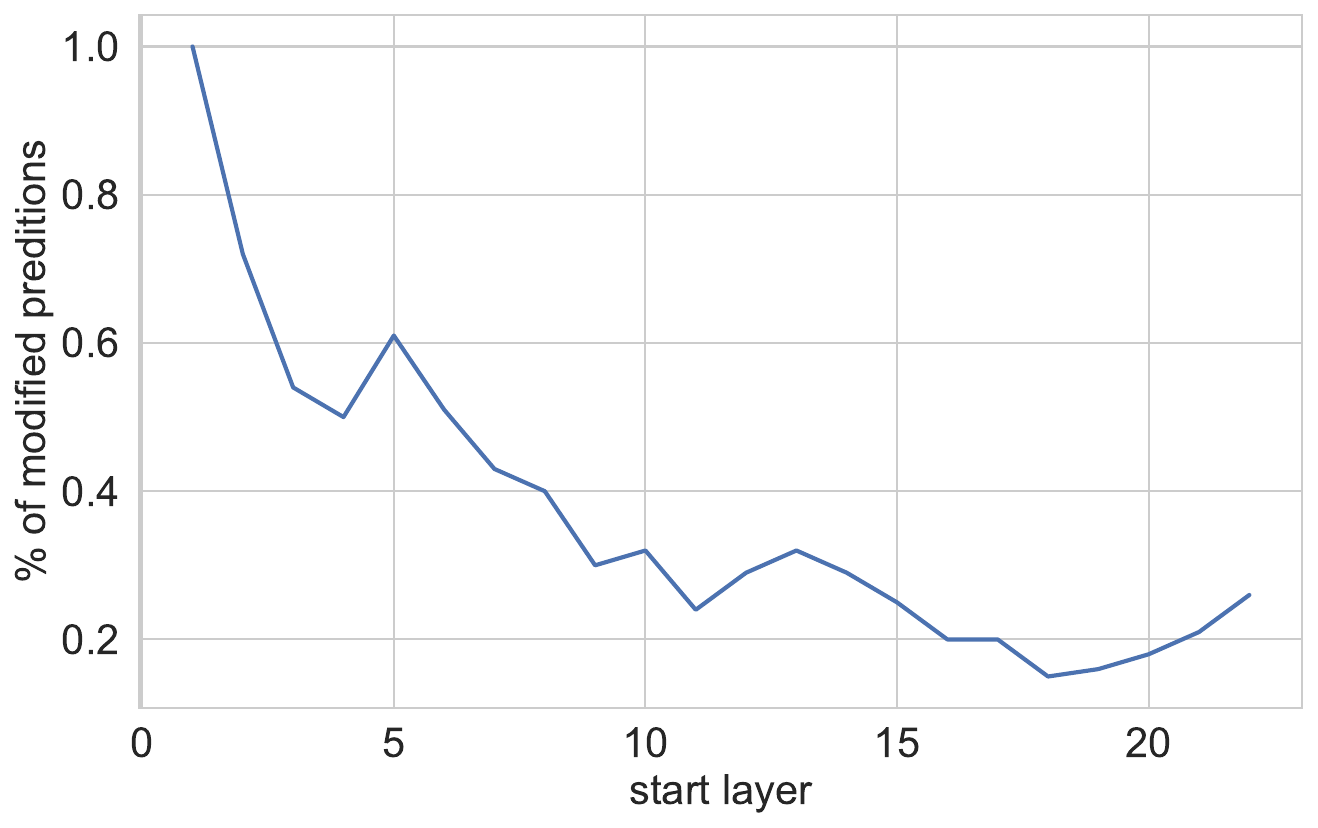}
    \caption{We visualize the effect of intervening in each 3-consecutive-layer range according to the procedure in~\S\ref{sec:knockouts}.}
    \label{fig:intervention_triplets}
\end{figure} 

In \S\ref{sec:knockouts}, we performed an intervention-based experiment to test the effect of zeroing out FFN sub-updates in layer computation. This produced a heat map of values corresponding to intervention's effect for each layer range. Here, we plot the intervention's effect across all 3-layer layer ranges. Note that since there are 24 layers, there are a total of 22 ranges of 3 consecutive layers. Fig.~\ref{fig:intervention_triplets} visualizes the effect of intervening in each such range.

As discussed in \S\ref{sec:knockouts}, we observe a steep drop in effect in the first ~10 layers, followed by a more leveled slope of decrease in the upper layers.

\subsection{Analyzing Changes in Rank and Probability of the Target Token}
In addition to measuring how often an intervention changes memorized predictions (\S\ref{sec:knockouts}), we further measure the average change in the rank and probability of the target token. Note that the original target rank for memorized predictions is always zero, as the target token is the top candidate in the original output distribution.

\paragraph{Change in target rank}
Fig.~\ref{figure:knowkout_rank}  shows the change in the target token's rank for all intervention experiments. 
Overall, we observe similar trends as in \S\ref{sec:knockouts}. First, zeroing out either dominant or non-dominant FFN sub-updates in upper layers (layers 11-24) does not affect memory recall, as the target token is still ranked as the top candidate in the output distribution. Moreover, zeroing out in early layers (1-10) damages memory recall as the target rank increases by $>20$ positions. Specifically, zeroing-out non-dominant FFN sub-updates in layers 2-4 increases the target rank by 60, and disabling either dominant or non-dominant sub-updates in the first layer completely eliminates memory recall as the rank increases to $>6000$.

\paragraph{Change in target probability}
Fig.~\ref{figure:knowkout_target_prob} shows the change in the target token's probability for all intervention experiments.
Unlike the prediction rank, which is mostly influenced by the lower layers during memory recall, the prediction probability is highly influenced by the intermediate and upper layers, where confidence boosting happens. Disabling FFN sub-updates in only three of these layers reduces the prediction probability by up to 33\%.
Considering the lower layers (1-9), zero-outs lead to a large probability decrease (up to 50\%). This is expected since these interventions often change the prediction, i.e. they eliminate the target from the top of the output distribution.

\begin{figure*}[t!]
    \begin{center}
    \includegraphics[scale=0.33]{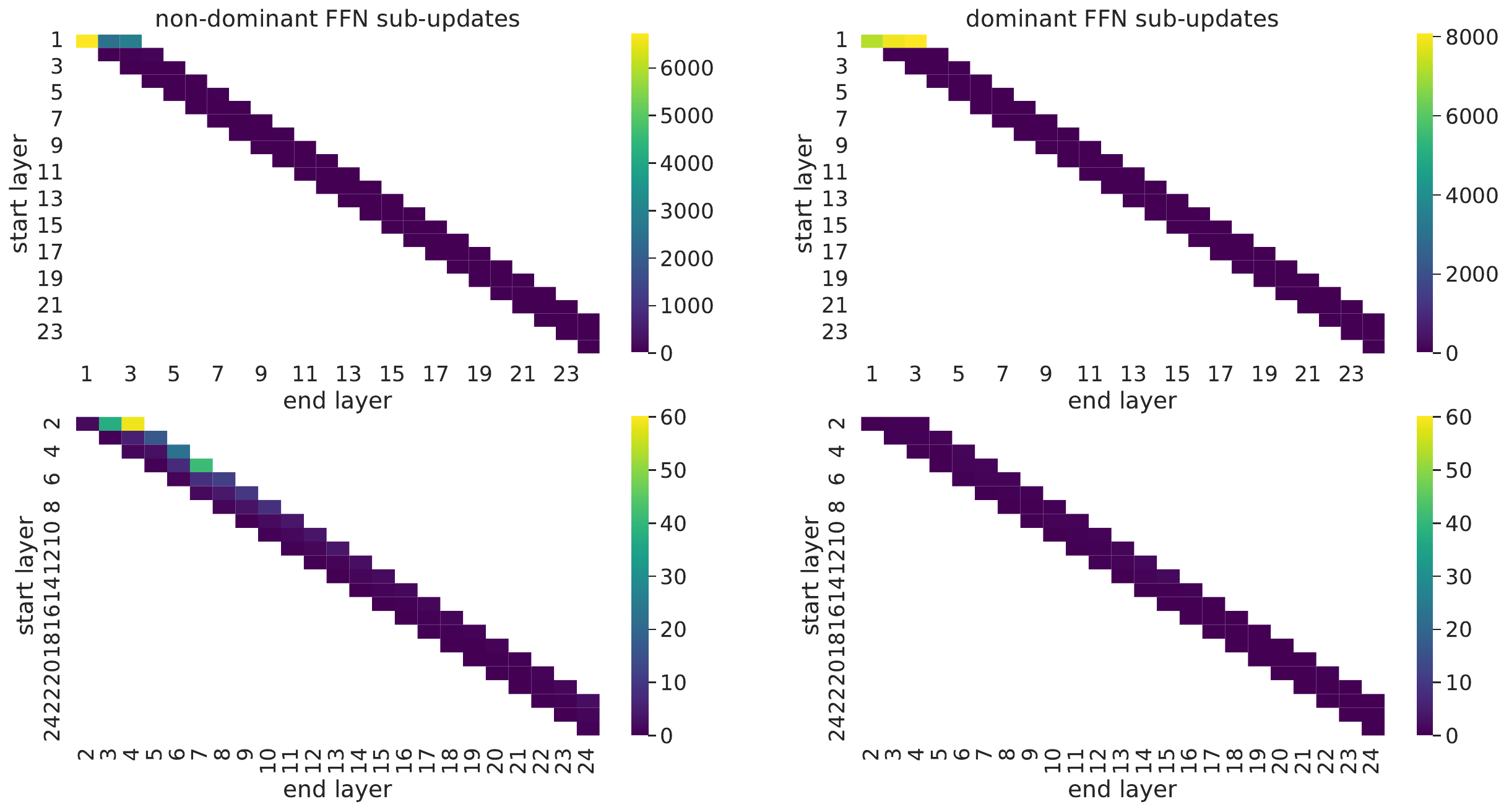}
    \end{center}
    \caption{Change in the rank of the target token following intervention zeroing out non-dominant (left) and dominant (right) FFN sub-updates in \gptm{}. Each cell shows the average change in rank of the target token after zeroing out the sub-updates in the layers between the start and end layers. For readability, we provide plots with the first layer (top) and without (bottom).}
    \label{figure:knowkout_rank}
\end{figure*}

\begin{figure*}[t!]
    \begin{center}
    \includegraphics[scale=0.33]{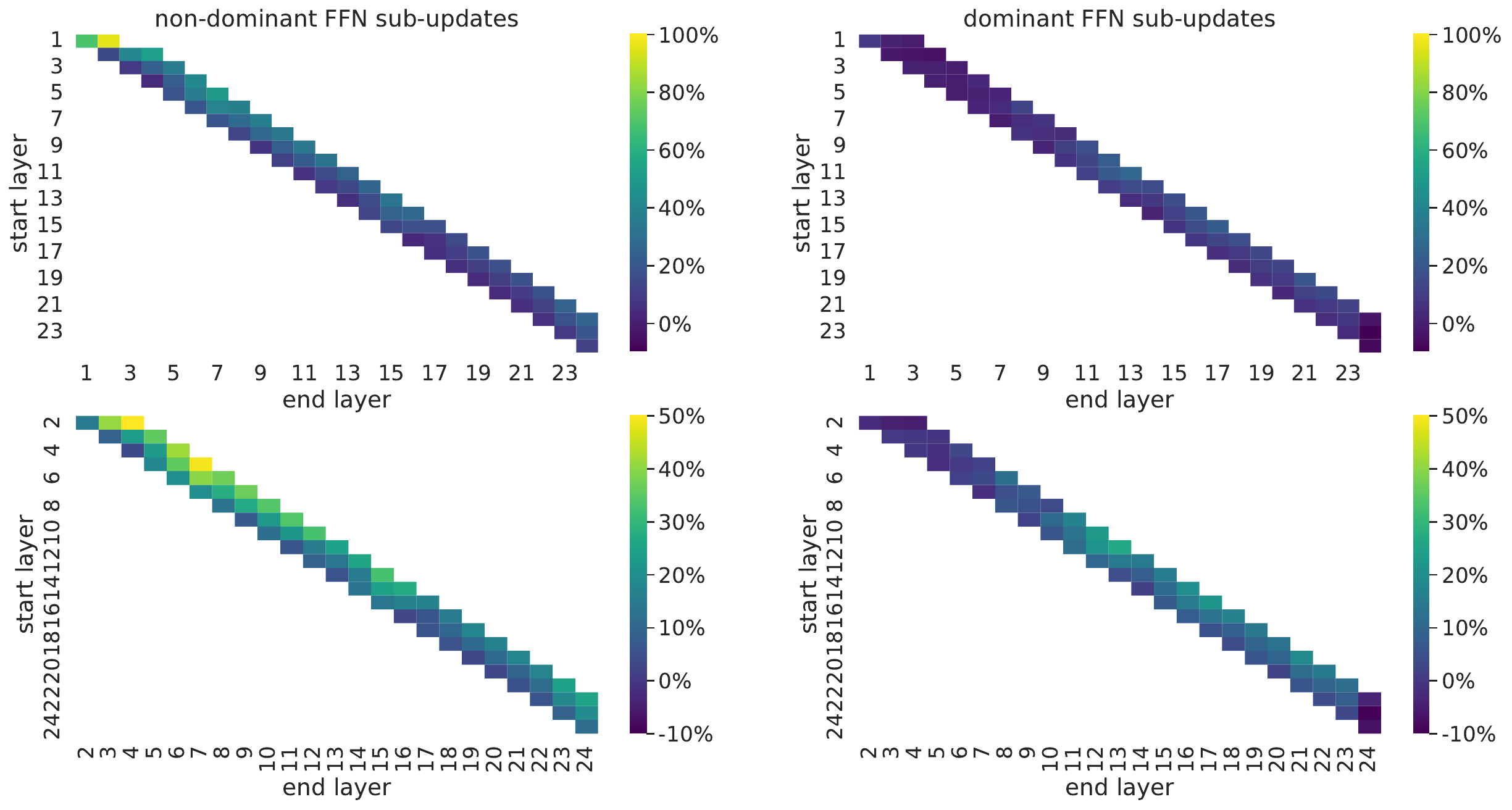}
    \end{center}
    \caption{Change in the probability of the target token following intervention zeroing out non-dominant (left) and dominant (right) FFN sub-updates in \gptm{}. Each cell shows the average change in the probability of the target token after zeroing out the sub-updates in the layers between the start and end layers. For readability, we provide plots with the first layer (top) and without (bottom).}
    \label{figure:knowkout_target_prob}
\end{figure*}

\section{Experimental Setting Details}
\label{apx:exp_settings}
Tab.~\ref{tab:models_hyperparameters} shows the evaluated models' hyperparameters. 

\begin{table*}[t]
    \centering
    \small
    \begin{tabular}{@{}lrrrr@{}}
    &  \textbf{\gptm{}} & \textbf{\gptl{}} & \textbf{\bertb{}} & \textbf{\bertl{}} \\
    \midrule
    Layers & 24 & 36 & 12 & 24 \\ 
    Model hidden dimensions ($d$) & 1024 & 1280 & 768 & 1024 \\ 
    Feed-forward dimensions ($d_m$) & 4096 & 5120 & 3072 & 4096 \\  
    Attention heads & 12 & 20 & 12 & 16 \\ 
    Parameters & 345M & 774M & 110M & 340M \\  
    Vocabulary size (\# of tokens) & 50,256 & 50,256 & 30,522 & 30,522 \\  
    \bottomrule
    \end{tabular}
    \caption{The models' hyperparameters.}
    \label{tab:models_hyperparameters}
\end{table*}

\end{document}